\documentclass{article}

\usepackage{subfigure}
\usepackage{authblk}

\usepackage{fullpage}

\usepackage{algo}
\usepackage{xcolor}

\usepackage{amsthm}
\usepackage{amssymb}
\usepackage{amsmath}
\usepackage{hyperref}

\usepackage{lmodern}
\usepackage[T1]{fontenc}

\usepackage{xspace}
\usepackage{enumerate}
\usepackage{boxedminipage}
\usepackage{graphicx}
\usepackage{paralist}

\allowdisplaybreaks

\newtheorem{theorem}{Theorem}
\newtheorem{lemma}[theorem]{Lemma}
\newtheorem{proposition}[theorem]{Proposition}

\theoremstyle{definition}

\renewenvironment{proof}{\vspace{-0.05in}\noindent{\bf Proof:}}%
        {\hspace*{\fill}$\Box$\par}
        {\hspace*{\fill}$\Box$\par}
        {\hspace*{\fill}$\Box$\par}

\newcommand{\E}{\mathbb{E}}

\newcommand{\R}{{\mathbb{R}}}

\def\script#1{\mathcal{#1}}

\def\set#1{\left\{#1\right\}}
\def\sep{\;|\;}

\def\Comment#1{\textsl{$\langle\!\langle$#1\/$\rangle\!\rangle$}}
\def\etal{{\em et al.}\xspace}

\def\norm#1{\left\|#1\right\|}
\def\grad{\nabla}
\def\Proj{\Pi}

\def\sA{\script{A}}

\def\sX{\script{X}}
\def\sY{\script{Y}}
\def\sP{\script{P}}
\def\sQ{\script{Q}}

\def\lovasz{Lov\'{a}sz\xspace}

\title{Random Coordinate Descent Methods for\\ Minimizing
Decomposable Submodular Functions}

\author[1]{Alina Ene}
\author[2]{Huy L. Nguy\~{\^{e}}n}
\affil[1]{Department of Computer Science and DIMAP \authorcr
			University of Warwick \authorcr {\tt
			A.Ene@dcs.warwick.ac.uk}}
\affil[2]{Simons Institute \authorcr
			University of California, Berkeley \authorcr {\tt
			hlnguyen@cs.princeton.edu}}

\begin{document}
\maketitle

\begin{abstract}
	Submodular function minimization is a fundamental optimization
	problem that arises in several applications in machine learning
	and computer vision. The problem is known to be solvable in
	polynomial time, but general purpose algorithms have high running
	times and are unsuitable for large-scale problems. Recent work
	have used convex optimization techniques to obtain very practical
	algorithms for minimizing functions that are sums of ``simple"
	functions. In this paper, we use random coordinate descent
	methods to obtain algorithms with faster \emph{linear}
	convergence rates and cheaper iteration costs.
	Compared to alternating projection
	methods, our algorithms do not rely on full-dimensional vector
	operations and they converge in significantly fewer iterations.
\end{abstract}

\section{Introduction}
\label{sec:intro}

Over the past few decades, there has been a significant progress on
minimizing submodular functions, leading to several polynomial time
algorithms for the problem
\cite{GrotschelLS81,Schrijver00,Iwata03,FleischerI03,Orlin09}.
Despite this intense focus, the running times of these algorithms are
high-order polynomials in the size of the data and designing faster
algorithms remains a central and challenging direction in submodular
optimization.

At the same time, technological advances have made it possible to
capture and store data at an ever increasing rate and level of
detail. A natural consequence of this ``big data" phenomenon is that
machine learning applications need to cope with data that is quite
large and it is growing at a fast pace. Thus there is an increasing
need for algorithms that are fast and scalable.

The general purpose algorithms for submodular minimization are
designed to provide worst-case guarantees even in settings where the
only structure that one can exploit is submodularity. At the other
extreme, graph cut algorithms are very efficient but they cannot
handle more general submodular functions. In many applications, the
functions strike a middle ground between these two extremes and it is
becoming increasingly more important to use their special structure
to obtain significantly faster algorithms.

Following \cite{Kolmogorov12,StobbeK10,JegelkaBS13,NishiharaJJ14}, we
consider the problem of minimizing \emph{decomposable} submodular
functions that can be expressed as a sum of \emph{simple} functions.
We use the term simple to refer to functions $F$ for which there is
an efficient algorithm for minimizing $F + w$, where $w$ is a linear
function. We assume that we are given black-box access to these
minimization procedures for simple functions.

Decomposable functions are a fairly rich class of functions and they
arise in several applications in machine learning and computer
vision. For example, they model higher-order potential functions for
MAP inference in Markov random fields, the cost functions in SVM
models for which the examples have only a small number of features,
and the graph and hypergraph cut functions in image segmentation.

The recent work of \cite{JegelkaBS13,Kolmogorov12,StobbeK10} has
developed several algorithms with very good empirical performance
that exploit the special structure of decomposable functions.  In
particular, \cite{JegelkaBS13} have shown that the problem of
minimizing decomposable submodular functions can be formulated as a
\emph{distance minimization} problem between two polytopes. This
formulation, when coupled with powerful convex optimization
techniques such as gradient descent or projection methods, it yields
algorithms that are very fast in practice and very simple to
implement \cite{JegelkaBS13}.

On the theoretical side, the convergence behaviour of these methods
is not very well understood. Very recently, Nishihara \etal
\cite{NishiharaJJ14} have made a significant progress in this
direction. Their work shows that the classical \emph{alternating
projections} method, when applied to the distance minimization
formulation, converges at a \emph{linear rate}.

{\bf Our contributions.}
In this work, we use random coordinate descent methods in order to
obtain algorithms for minimizing decomposable submodular functions
with faster convergence rates and cheaper iteration costs. We analyze
a standard and an accelerated random coordinate descent algorithm and
we show that they achieve linear convergence rates. Compared to
alternating projection methods, our algorithms do not rely on
full-dimensional vector operations and they are faster by a factor
equal to the number of simple functions. Moreover, our accelerated algorithm converges in a much smaller
number of iterations. We experimentally evaluate our algorithms on
image segmentation tasks and we show that they perform very well and
they converge much faster than the alternating projection method.

{\bf Submodular minimization.}
The first polynomial time algorithm for submodular optimization was
obtained by Gr\"{o}tschel \etal \cite{GrotschelLS81} using the
ellipsoid method. There are several combinatorial algorithms for the
problem \cite{Schrijver00,Iwata03,FleischerI03,Orlin09}. Among the
combinatorial methods, Orlin's algorithm \cite{Orlin09} achieves the
best time complexity of $O(n^5 T + n^6)$, where $n$ is the size of
the ground set and $T$ is the maximum amount of time it takes to
evaluate the function. Several algorithms have been proposed for
minimizing decomposable submodular functions
\cite{StobbeK10,Kolmogorov12,JegelkaBS13,NishiharaJJ14}. Stobbe and
Krause \cite{StobbeK10} use gradient descent methods with sublinear
convergence rates for minimizing sums of concave functions applied to
linear functions. Nishihara \etal \cite{NishiharaJJ14} give an
algorithm based on alternating projections that achieves a linear
convergence rate. 

\subsection{Preliminaries and Background}
\label{sec:prelim}

Let $V$ be a finite ground set of size $n$; without loss
of generality, $V = \set{1, 2, \dots, n}$. We view each point $w \in
\R^n$ as a modular set function $w(A) = \sum_{i \in A} w_i$ on the
ground set $V$.

A set function $F: 2^V \rightarrow \R$ is \emph{submodular} if $F(A)
+ F(B) \geq F(A \cap B) + F(A \cup B)$ for any two sets $A, B
\subseteq V$.
A set function $F_i: 2^V \rightarrow \R$ is \emph{simple} if there is
a fast subroutine for minimizing $F_i + w$ for any modular function
$w \in \R^n$.

In this paper, we consider the problem of minimizing a submodular
function $F: 2^V \rightarrow \R$ of the form $F = \sum_{i = 1}^r
F_i$, where each function $F_i$ is a simple submodular set function:
\begin{equation}
	\min_{A \subseteq V} F(A) \equiv \min_{A \subseteq V} \sum_{i =
	1}^r F_i(A) \tag{DSM}
\end{equation}
We assume without loss of generality that the function $F$ is
normalized, i.e., $F(\emptyset) = 0$. Additionally, we assume we are
given black-box access to oracles for minimizing $F_i + w$ for each
function $F_i$ in the decomposition and each $w \in \R^n$.

The \emph{base polytope} $B(F)$ of $F$ is defined as follows.
\begin{equation*}
\begin{split}
	B(F) &= \{w \in \R^n \sep w(A) \leq F(A)
	\; \text{ for all } A \subseteq V,\\
	&\qquad\qquad\qquad w(V) = F(V)\} 
\end{split}
\end{equation*}
The discrete problem (DSM)\footnote{DSM stands for decomposable
submodular function minimization.} admits an exact convex programming
relaxation based on the \lovasz extension of a submodular function.
The Lov\'{a}sz extension $f$ of $F$ can be written as the support
function of the base polytope $B(F)$:
\begin{equation*} 
	f(x) = \max_{w \in B(F)} \langle w, x \rangle \;\; \forall x \in
	\R^n
\end{equation*}
Even though the base polytope $B(F)$ has exponentially many vertices,
the Lov\'{a}sz extension $f$ can be evaluated efficiently using the
greedy algorithm of Edmonds (see for example \cite{Schrijver-book}).
Given any point $x \in \R^n$, Edmonds' algorithm evaluates $f(x)$
using $O(n \log{n}) \times T$ time, where $T$ is the time needed to
evaluate the submodular function $F$.

Lov\'{a}sz showed that a set function $F$ is submodular if and only
if its Lov\'{a}sz extension $f$ is convex \cite{Lovasz83}. Thus we
can relax the problem of minimizing $F$ to the following non-smooth
convex optimization problem:
\begin{equation*}
	\min_{x \in [0, 1]^n} f(x) \equiv \min_{x \in [0, 1]^n}
	\sum_{i = 1}^r f_i(x)
\end{equation*}
where $f_i$ is the Lov\'{a}sz extension of $F_i$.

The relaxation above is exact. Given a fractional solution $x$ to the
\lovasz Relaxation, the best threshold set of $x$ has cost at most
$f(x)$. 

An important drawback of the \lovasz relaxation is that its objective
function is not smooth. Following previous work
\cite{JegelkaBS13,NishiharaJJ14}, we
consider a proximal version of the problem ($\norm{\cdotp}$ denotes
the $\ell_2$-norm):
\begin{equation} \label{eq:proximal}
	\min_{x \in \R^n} \left( f(x) + {1 \over 2} \norm{x}^2
	\right) \equiv \min_{x \in \R^n} \sum_{i = 1}^r
	\left(f_i(x) + {1 \over 2r} \norm{x}^2 \right)
	\tag{Proximal}
\end{equation}
Given an optimal solution $x$ to the proximal problem $\min_{x \in
\R^n} \big(f(x) + {1 \over 2} \norm{x}^2 \big)$, we can construct an
optimal solution to the discrete problem (DSM) by thresholding $x$ at
zero; more precisely, the set $\set{v \in V \colon x(v) \geq 0}$ is
an optimal solution to (DSM) (Proposition~{8.6} in
\cite{Bach-monograph}).
\begin{lemma}[\cite{JegelkaBS13}] \label{lem:dual-prox}
	The dual of the proximal problem
	\begin{equation*}
		\min_{x \in \R^n} \sum_{i = 1}^r \left(f_i(x) + {1 \over
		2r} \norm{x}^2 \right)
	\end{equation*}
	is the problem
	\begin{equation*}
		\max_{y^{(1)} \in B(F_1), \dots, y^{(r)} \in B(F_r)} - {1
		\over 2} \norm{\sum_{i = 1}^r y^{(i)}}^2
	\end{equation*}
	The primal and dual variables are linked as $x = - \sum_{i = 1}^r
	y^{(i)}$.
\end{lemma}

\noindent
Lemma~\ref{lem:dual-prox} was proved in \cite{JegelkaBS13}; we
include a proof in Section~\ref{app:lem-dual-prox} for completeness.

We write the dual proximal problem in the following equivalent form:
\begin{equation}
	\min_{y^{(1)} \in B(F_1), \dots, y^{(r)} \in B(F_r)}
	\norm{\sum_{i = 1}^r y^{(i)}}^2 \tag{Prox-DSM}
\end{equation}
It follows from the discussion above that, given an optimal solution
$y = (y^{(1)}, \dots, y^{(r)})$ to (Prox-DSM), we can recover an
optimal solution to (DSM) by thresholding $x = - \sum_{i = 1}^r
y^{(i)}$ at zero.

\section{Random Coordinate Descent Algorithm}
\label{sec:rcdm}

\begin{figure}[t]
\begin{algo}
\underline{\bf RCDM Algorithm for (Prox-DSM)}
\\\Comment{We can take the initial point $y_0$ to be $0$}
\\ Start with $y_0 = (y^{(1)}_0, \dots, y^{(r)}_0) \in \sY$
\\ In each iteration $k$ ($k \geq 0$)
\\\; Pick an index $i_k \in \set{1, 2, \dots, r}$ uniformly at random
\\\; \Comment{Update the block $i_k$}
\\\; $y_{k + 1}^{(i_k)} \leftarrow \underset{y \in
B(F_{i_k})}{\argmin} \Big( \left\langle \grad_{i_k} g(y_k), y - y_k^{(i_k)}
\right\rangle$
\\\>\>\>\>\> $ + {L_{i_k} \over 2} \norm{y - y_k^{(i_k)}}^2 \Big)$
\end{algo}
\caption{Random block coordinate descent method for (Prox-DSM). It
finds a solution to (Prox-DSM) given access to an oracle for $\min_{y
\in B(F_i)} \left( \langle y, a \rangle + \norm{y}^2 \right)$.}
\label{fig:rcdm}
\end{figure}

In this section, we give an algorithm for the problem (Prox-DSM) that
is based on the random coordinate gradient descent method (RCDM) of
\cite{Nesterov10}. The algorithm is given in Figure~\ref{fig:rcdm}.
The algorithm is very easy to implement and it uses oracles for
problems of the form $\min_{y \in B(F_i)} \left( \langle y, a \rangle
+ \norm{y}^2 \right)$, where $i \in [r]$ and $a \in \R^n$. Since each
function $F_i$ is simple, we have such oracles that are very
efficient.

In the remainder of this section, we analyze the convergence rate of
the RCDM algorithm. We emphasize that the objective function of
(Prox-DSM) is \emph{not} strongly convex and thus we cannot use as a
black-box Nesterov's analysis of the RCDM method for minimizing
strongly convex functions. Instead, we exploit the special structure
of the problem to achieve convergence guarantees that match the rate
achievable for strong convex objectives with strong convexity
parameter $1 / (n^2 r)$. Our analysis shows that the RCDM algorithm
is faster by a factor of $r$ than the alternating projections
algorithm from \cite{NishiharaJJ14}.

\textbf{Outline of the analysis:}
Our analysis has two main components. First, we build on the work of
\cite{NishiharaJJ14} in order to prove a key theorem
(Theorem~\ref{thm:projection}). This theorem exploits the special
structure of the (Prox-DSM) problem and it allows us to overcome the
fact that the objective function of (Prox-DSM) is not strongly
convex. Second, we modify Nesterov's analysis of the RCDM algorithm
for minimizing strongly convex functions and we replace the strong
convexity guarantee by the guarantee given by
Theorem~\ref{thm:projection}.

We start by introducing some notation; for the most part, we follow
the notation of \cite{Nesterov10} and \cite{ NishiharaJJ14}. Let
$\R^{nr} = \bigotimes_{i = 1}^r \R^n$. We write a vector $y \in
\R^{nr}$ as $y = (y^{(1)}, \dots, y^{(r)})$, where each block
$y^{(i)}$ is an $n$-dimensional vector.
Let $\sY = \bigotimes_{i = 1}^r B(F_i)$ be the constraint set of
(Prox-DSM). Let $g: \R^{nr} \rightarrow \R$ be the objective function
of (Prox-DSM): $g(y) = \norm{\sum_{i = 1}^r y^{(i)}}^2$. We use
$\grad g$ to denote the gradient of $g$, i.e., the
$(nr)$-dimensional vector of partial derivatives. For each $i \in
\set{1, \dots, r}$, we use $\grad_i g(y) \in \R^n$ to denote the
$i$-th block of coordinates of $\grad g(y)$.

Let $S \in \R^{n \times nr}$ be the following
matrix:
\[
S = {1 \over \sqrt{r}} \Big[ \underbrace{I_n I_n\cdots
I_n}_{r\textnormal{ times}} \Big]
\]
Note that $g(y) = r \norm{S y}^2$ and $\grad g(y) = 2r S^T Sy$.
Additionally, for each $i \in \set{1, 2, \dots, r}$, $\grad_i g$ is
Lipschitz continuous with constant $L_i = 2$:
\begin{equation} \label{eq:grad-lipschitz}
	\norm{\grad_i g(x) - \grad_i g(y)} \leq L_i \norm{x^{(i)} -
	y^{(i)}},
\end{equation}
for all vectors $x, y \in \R^{nr}$ that differ only in block $i$.

Our first step is to prove the following key theorem that builds on
the work of \cite{NishiharaJJ14}.

\begin{theorem} \label{thm:projection}
	Let $y \in \sY$ be a feasible solution to (Prox-DSM). Let $y^*$
	be an optimal solution to (Prox-DSM) that minimizes $\norm{y -
	y^*}$. We have
		\[ \norm{S(y - y^*)} \geq {1 \over n r} \norm{y -
		y^*}.\]
\end{theorem}

\noindent
The proof of Theorem~\ref{thm:projection} uses the following key
result from \cite{NishiharaJJ14arxiv}. We will need the following
definitions from \cite{NishiharaJJ14arxiv}.

Let $d(K_1, K_2) = \inf\set{\norm{k_1 - k_2} \colon k_1 \in K_1, k_2
\in K_2}$ be the distance between sets $K_1$ and $K_2$. Let $\sP$ and
$\sQ$ be two closed convex sets in $\R^d$. Let $E
\subseteq \sP$ and $H \subseteq \sQ$ be the sets of closest points
\begin{align*}
  E &= \set{p \in \sP \colon d(p, \sQ) = d(\sP, \sQ)}\\		
  H &= \set{q \in \sQ \colon d(q, \sP) = d(\sP, \sQ)}
\end{align*}
Since $\sP$ and $\sQ$ are convex, for each point in $p \in E$, there
is a unique point $q \in H$ such that $d(p, q) = d(\sP, \sQ)$ and
vice versa. Let $v = \Proj_{\sQ - \sP} 0$; note that $H = E + v$. Let
$\sQ' = \sQ - v$; $\sQ'$ is a translated version of $\sQ$ and it
intersects $\sP$ at $E$. Let
	\[ \kappa_* = \sup_{x \in (\sP \cup \sQ') \setminus E} {d(x, E)
	\over \max\set{d(x, \sP), d(x, \sQ')}}.\]
By combining Corollary~{5} and Proposition~{11} from
\cite{NishiharaJJ14arxiv}, we obtain the following theorem.

\begin{theorem}[\cite{NishiharaJJ14}] \label{thm:kappa-upper-bound}
	If $\sP$ is the polyhedron $\bigotimes_{i = 1}^r B(F_i)$ and
	$\sQ$ is the polyhedron $\set{y \in \R^{nr} \colon \sum_{i = 1}^r
	y^{(i)} = 0}$, we have $\kappa_* \leq n r$.
\end{theorem}

Now we are ready to prove Theorem~\ref{thm:projection}. Let
\begin{align*}
	\sP &= \bigotimes\nolimits_{i = 1}^r B(F_i) = \sY\\
	\sQ &= \set{y \in \R^{nr} \colon \sum\nolimits_{i = 1}^r y^{(i)} = 0} =
	\set{y \in \R^{nr} \colon S y = 0}
\end{align*}
We define $\sQ'$ and $\kappa_*$ as above.

Let $y$ and $y^*$ be the two points in the statement of the theorem.
Note that $y \in \sP$ and $y^* \in E$, since $E$ is the set of all
optimal solutions to (Prox-DSM) (see Proposition~\ref{prop:optimal-E}
in Section~\ref{app:rcdm} for a proof). We may assume that $y \notin
E$, since otherwise the theorem trivially holds. Since $y \in \sP
\setminus E$, we have
	\[ \kappa_* \geq {d(y, E) \over d(y, \sQ')} \]
Since $y^*$ is an optimal solution that is closest to $y$, we have
$d(y, E) = \norm{y - y^*}$. Using the fact that the rows of $S$ form
a basis for the orthogonal complement of $\sQ$, we can show that
$d(y, \sQ') = \norm{S(y - y^*)}$ (see
Proposition~\ref{prop:distance-to-Q'} in Section~\ref{app:rcdm} for a
proof). Therefore
	\[ \kappa_* \geq {\norm{y - y^*} \over \norm{S(y - y^*)}}.\]
Theorem~\ref{thm:projection} now follows from
Theorem~\ref{thm:kappa-upper-bound}.

In the remainder of this section, we use Nesterov's analysis
\cite{Nesterov10} in conjunction with Theorem~\ref{thm:projection} in
order to show that the RCDM algorithm converges at a linear rate.
Recall that $E$ is the set of all optimal solutions to (Prox-DSM).

\begin{theorem} \label{thm:rcdm-convergence}
	After $(k + 1)$ iterations of the RCDM algorithm, we have
		\[ \E\left[ d(y_{k}, E)^2 + g(y_{k + 1}) - g(y^*) \right]
		\leq \left(1 - {2 \over n^2r^2 + r} \right)^{k + 1} \left(
		d(y_0, E)^2 + g(y_0) - g(y^*) \right), \]
	where $y^* = \argmin_{y \in E} \norm{y - y_k}$ is the optimal
	solution that is closest to $y_k$.
\end{theorem} 

We devote the rest of this section to the proof of
Theorem~\ref{thm:rcdm-convergence}. We recall the following
well-known lemma, which we refer to as the first-order optimality
condition.

\begin{lemma}[Theorem~{2.2.5} in \cite{Nesterov-book}]
\label{lem:first-order-optimality}
	Let $f: \R^d \rightarrow \R$ be a differentiable convex function
	and let $Q \subseteq \R^d$ be a closed convex set. A point $x^*
	\in \R^d$ is a solution to the problem $\min_{x \in Q} f(x)$ if
	and only if
		\[ \langle \grad f(x^*), x - x^* \rangle \geq 0 \]
	for all $x \in Q$.
\end{lemma}

It follows from the first-order optimality condition for
$y^{(i_k)}_{k + 1}$ that, for any $z \in B(F_{i_k})$,
\begin{equation}\label{eqn:iter-opt}
	\left\langle \grad_{i_k} g(y_k) + L_{i_k} \left( y^{(i_k)}_{k+1}
	- y^{(i_k)}_k \right), z - y^{(i_k)}_{k+1} \right\rangle \ge 0
\end{equation}
We have
\begin{align}
	g(y_{k + 1}) &= g(y_k) + \int_0^1 \langle y_{k + 1} - y_k, \grad
	g(y_k + t(y_{k + 1} - y_k))\rangle dt \nonumber\\
	&= g(y_{k}) + \langle \grad g(y_{k}), y_{k+1} - y_{k}\rangle +
	\int_0^1 \big\langle y_{k+1} - y_{k}, \grad g(y_{k} + t(y_{k+1} -
	y_{k})) - \grad g(y_k)\big\rangle dt\nonumber\\
	&= g(y_{k}) + \left\langle \grad_{i_k} g(y_{k}), y_{k+1}^{(i_k)}
	- y_{k}^{(i_k)} \right\rangle + \int_0^1 \big\langle
	y_{k+1}^{(i_k)} - y_{k}^{(i_k)}, \grad_{i_k} g(y_{k} + t(y_{k+1}
	- y_{k})) - \grad_{i_k} g(y_k)\big\rangle dt\nonumber\\
	&\le g(y_{k}) + \left\langle \grad_{i_k} g(y_{k}),
	y_{k+1}^{(i_k)} - y_{k}^{(i_k)} \right\rangle + \int_0^1 \norm{
	y_{k+1}^{(i_k)} - y_{k}^{(i_k)}} \norm{\grad_{i_k} g(y_{k} +
	t(y_{k+1} - y_{k})) - \grad_{i_k} g(y_k)} dt\nonumber\\
	&\overset{(\ref{eq:grad-lipschitz})}{\le} g(y_{k}) + \left\langle
	\grad_{i_k} g(y_{k}), y_{k+1}^{(i_k)} - y_{k}^{(i_k)}
	\right\rangle + \int_0^1 L_{i_k} \norm{ y_{k+1}^{(i_k)} -
	y_{k}^{(i_k)}}^2 t dt\nonumber\\
	&= g(y_{k}) + \left\langle \grad_{i_k} g(y_{k}), y_{k+1}^{(i_k)}
	- y_{k}^{(i_k)} \right\rangle + \frac{L_{i_k}}{2}
	\norm{y_{k+1}^{(i_k)} - y_{k}^{(i_k)}}^2 \label{eqn:improve-f}
\end{align}
On the third line, we have used the fact that $y_k$ and $y_{k + 1}$
agree on all coordinate blocks except the $i_k$-th block. On the
fourth line, we have used the Cauchy-Schwartz inequality. On the
fifth line, we have used inequality (\ref{eq:grad-lipschitz}).

Let $y^* = \argmin_{y \in E} \norm{y - y_k}$ be the optimal solution
that is closest to $y_{k}$. We have
\begin{align}
	\norm{y_{k+1} - y^*}^2 &= \norm{y_{k} - y^*}^2 + \norm{y_{k+1} -
	y_{k}}^2 + 2\langle y_{k}-y^*, y_{k+1} - y_{k}\rangle \nonumber\\
	&= \norm{y_{k} - y^*}^2 - \norm{y_{k+1} - y_{k}}^2 + 2
	\left\langle y_{k+1} - y^{*}, y_{k+1} - y_{k} \right\rangle
	\nonumber\\
	&= \norm{y_{k} - y^*}^2 - \norm{y_{k+1}^{(i_k)} -
	y_{k}^{(i_k)}}^2 + 2 \left\langle
	y_{k+1}^{(i_k)}-(y^{*})^{(i_k)}, y_{k+1}^{(i_k)} - y_{k}^{(i_k)}
	\right\rangle \nonumber\\
	&\overset{(\ref{eqn:iter-opt})}{\le} \norm{y_{k} - y^*}^2 -
	\norm{y_{k+1}^{(i_k)} - y_{k}^{(i_k)}}^2 + {2 \over L_{i_k}}
	\left\langle \grad_{i_k} g(y_{k}), (y^*)^{(i_k)} -
	y_{k+1}^{(i_k)} \right\rangle \nonumber\\
	&= \norm{y_{k} - y^*}^2 + {2 \over L_{i_k}} \left\langle
	\grad_{i_k} g(y_{k}), (y^*)^{(i_k)} - y_{k}^{(i_k)}
	\right\rangle \nonumber\\
	&\qquad\qquad\qquad 
	-{2 \over L_{i_k}} \left({L_{i_k} \over 2} \norm{y_{k+1}^{(i_k)}
	- y_{k}^{(i_k)}}^2 + \left\langle \grad_{i_k} g(y_{k}),
	y_{k+1}^{(i_k)} - y_{k}^{(i_k)} \right\rangle\right) \nonumber\\
	&\overset{(\ref{eqn:improve-f})}{\le} \norm{y_{k} - y^*}^2 + {2
	\over L_{i_k}} \left\langle \grad_{i_k} g(y_{k}), (y^*)^{(i_k)} -
	y_{k}^{(i_k)} \right\rangle - {2 \over L_{i_k}} \left(g(y_{k+1})
	- g(y_{k})\right) \label{eqn:dist-to-E}
\end{align}
On the third line, we have used the fact that $y_k$ and $y_{k + 1}$
agree on all coordinate blocks except the $i_k$-th block. On the
fourth line, we have used the inequality (\ref{eqn:iter-opt}) with $z
= (y^*)^{(i_k)}$. On the last line, we have used inequality
(\ref{eqn:improve-f}). 

If we rearrange the terms of the inequality (\ref{eqn:dist-to-E}),
take expectation over $i_k$, and substitute $L_{i_k} = 2$, we obtain
\begin{align}
	\E_{i_k}\left[\norm{y_{k+1} - y^*}^2 + g(y_{k+1}) - g(y^*)
	\right] &\le \norm{y_{k} - y^*}^2 + g(y_{k}) - g(y^*) +
	\frac{1}{r}\langle \grad g(y_{k}), y^* - y_{k}\rangle
	\label{eqn:expected-dist}
\end{align}
We can upper bound $\langle \grad g(y_{k}), y^* - y_{k}\rangle$ as
follows.
\begin{align}
	\left\langle \grad g(y_{k}), y^* - y_{k} \right\rangle &= 2r
	\left\langle S^T S y_k, y^* - y_k \right\rangle \nonumber\\
	&=  r \left\langle S^T S y_k + S^T S y^*, y^* - y_k \right\rangle
	+ r \left\langle S^T S y_k - S^T S y^*, y^* - y_k \right\rangle
	\nonumber\\
	&=  r \left\langle S^T S y_k + S^T S y^*, y^* - y_k
	\right\rangle - r \norm{S(y_k - y^*)}^2 \nonumber\\ 
	&=  r \left\langle S (y_k + y^*), S(y^* - y_k)
	\right\rangle - r \norm{S(y_k - y^*)}^2 \nonumber\\ 
	&= (g(y^*) - g(y_k)) - r \norm{S(y_k - y^*)}^2 \nonumber\\
	&\le (g(y^*) - g(y_k)) - {1 \over n^2 r} \norm{y_k - y^*}^2
	\qquad \mbox{(By Theorem~\ref{thm:projection})}
	\label{eqn:first-bound}
\end{align}
On the first and fifth lines, we have used the fact that $\grad g(z)
= 2r S^T Sz$ and $g(z) = r \norm{Sz}^2$ for any $z \in \R^{nr}$. On
the last line, we have used Theorem~\ref{thm:projection}.

Since $y^*$ is an optimal solution to (Prox-DSM), the first-order
optimality condition gives us that
\begin{align}
	\langle \grad g(y^*), y^* - y_k \rangle = 2r \langle S^TS y^*, y^* -
	y_k \rangle \le 0 \label{eqn:opty}
\end{align}
Using the inequality above, we can also upper bound $\langle \grad
g(y_{k}), y^* - y_{k}\rangle$ as follows.  
\begin{align}
	\langle \grad g(y_{k}), y^* - y_{k} \rangle &= 2r \langle S^T S
	y_k, y^* - y_k \rangle \nonumber\\
	&= 2r\langle S^T S y^*, y^* - y_k \rangle + 2 r\langle S^T S y_k -
	S^T S y^*, y^* - y_k \rangle \nonumber\\
	&= 2r\langle S^T S y^*, y^* - y_k \rangle  - 2r \norm{S(y_k -
	y^*)}^2 \nonumber\\
	&\overset{(\ref{eqn:opty})}{\le} - 2r \norm{S(y_k -
	y^*)}^2 \nonumber\\
	&\le  - {2 \over n^2 r} \norm{y_k - y^*}^2 \qquad \mbox{(By
	Theorem~\ref{thm:projection})} \label{eqn:second-bound}
\end{align}
By taking ${2 \over n^2
r + 1} \times (\ref{eqn:first-bound}) + \left(1 - {2 \over n^2r + 1}
\right) \times (\ref{eqn:second-bound})$, we obtain
\begin{align}
	\left\langle \grad g(y_{k}), y^* - y_{k} \right\rangle &\le - {2
	\over n^2r + 1} \left(g(y_k) - g(y^*) + \norm{y_k - y^*}^2
	\right) \label{eqn:combined-bound}
\end{align}
By (\ref{eqn:expected-dist}) and (\ref{eqn:combined-bound}),
\begin{align*}
	\underset{i_k}{\E}\left[\norm{y_{k+1} - y^*}^2 + g(y_{k+1}) -
	g(y^*) \right] &\le \left(1 - {2 \over n^2 r^2 + r} \right)
	\left(g(y_k) - g(y^*) + \norm{y_k - y^*}^2 \right)
\end{align*}
Note that $d(y_{k+1}, E)^2 \le \norm{y_{k+1} - y^*}^2$ and $d(y_k,
E)^2 = \norm{y_k - y^*}^2$. Therefore
\begin{align*}
	\underset{i_k}{\E}\left[ d(y_{k+1}, E)^2 + g(y_{k+1}) - g(y^*)
	\right] &\le \left(1 - {2 \over n^2r^2 + r} \right) \left( d(y_k,
	E)^2 + g(y_{k}) - g(y^*) \right)
\end{align*}
By taking expectation over $\xi = (i_1, \dots, i_k)$, we get
\begin{align*}
	\underset{\xi}{\E}\left[ d(y_{k+1}, E)^2 + g(y_{k+1}) - g(y^*)
	\right] &\leq \left(1 - {2 \over n^2r^2 + r}
	\right)^{k + 1} \left( d(y_0, E)^2 + g(y_0) - g(y^*) \right)
\end{align*}
Therefore the proof of Theorem~\ref{thm:rcdm-convergence} is complete.

\section{Accelerated Coordinate Descent Algorithm}
\label{sec:acdm}

\begin{figure}[t]
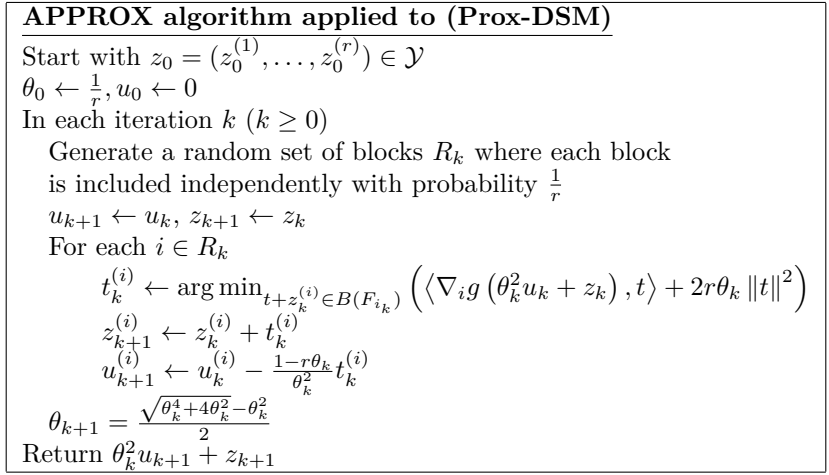

\begin{algo}
\underline{\bf APPROX algorithm applied to (Prox-DSM)}
\\ Start with $z_0 = (z^{(1)}_0, \dots, z^{(r)}_0) \in \sY$
\\ $\theta_0 \leftarrow {1 \over r}, u_0 \leftarrow 0$
\\ In each iteration $k$ ($k \geq 0$)
\\\> Generate a random set of blocks $R_k$ where each block
\\\> is included independently with probability ${1 \over r}$
\\\> $u_{k + 1} \leftarrow u_k$, $z_{k + 1} \leftarrow z_k$
\\\> For each $i \in R_k$
\\\>\> $t_k^{(i)} \leftarrow \argmin_{t + z_k^{(i)} \in B(F_{i_k})}
\left( \left\langle \grad_{i} g\left(\theta_k^2 u_k + z_k \right), t
\right\rangle + 2 r\theta_k \norm{t}^2 \right)$
\\\>\> $z_{k + 1}^{(i)} \leftarrow z_k^{(i)} + t_k^{(i)}$
\\\>\> $u_{k + 1}^{(i)} \leftarrow u_k^{(i)} - {1 - r\theta_k \over
\theta_k^2} t_k^{(i)}$
\\\> $\theta_{k + 1} = {\sqrt{\theta_k^4 + 4\theta_k^2} - \theta_k^2
\over 2}$
\\ Return $\theta_k^2 u_{k + 1} + z_{k + 1}$
\end{algo}
\caption{The APPROX algorithm of \cite{FR13} applied to (Prox-DSM).
It finds a solution to (Prox-DSM) given access to an oracle for
$\min_{y \in B(F_i)} \left( \langle y, a \rangle + \norm{y}^2
\right)$.}
\label{fig:approx}
\end{figure}

\begin{figure}[t]
\begin{algo}
\underline{{\bf ACDM Algorithm for (Prox-DSM)}}
\\\Comment{We can take the initial point $y_0$ to be $0$}
\\ Start with $y_0 = (y^{(1)}_0, \dots, y^{(r)}_0) \in \sY$
\\ In each epoch $\ell$ ($\ell \geq 0$)
\\\> Run the algorithm in Figure~\ref{fig:approx} for $(4 n r^{3/2} +
1)$
\\\> iterations with $y_{\ell}$ as its starting point ($z_0 =
y_{\ell})$
\\\> Let $y_{\ell + 1}$ be the vector returned by the algorithm
\end{algo}
\caption{Accelerated block coordinate descent method for (Prox-DSM).
It finds a solution to (Prox-DSM) given access to an oracle for
$\min_{y \in B(F_i)} \left( \langle y, a \rangle + \norm{y}^2
\right)$.}
\label{fig:acdm}
\end{figure}

In this section, we give an accelerated random coordinate descent
(ACDM) algorithm for (Prox-DSM). The algorithm uses the APPROX
algorithm of Fercoq and Richt\'{a}rik \cite{FR13} as a subroutine.
The APPROX algorithm (Algorithm~{2} in \cite{FR13}), when applied to
the (Prox-DSM) problem, yields the algorithm in
Figure~\ref{fig:approx}. The ACDM algorithm runs in a sequence of
epochs (see Figure~\ref{fig:acdm}). In each epoch, the algorithm
starts with the solution of the previous epoch and it runs the APPROX
algorithm for $\Theta(n r^{3/2})$ iterations. The solution
constructed by the APPROX algorithm will be the starting point of the
next epoch. Note that, for each $i$, the gradient $\grad_i g(y) =
2\sum_j y^{(j)}$ can be easily maintained at a cost of $O(n)$ per block
update, and thus the iteration cost is dominated by the time to
compute projection.

In the remainder of this section, we use the analysis of \cite{FR13}
together with Theorem~\ref{thm:projection} in order to show that the
ACDM algorithm converges at a linear rate. We follow the notation
used in Section~\ref{sec:rcdm}. 

\begin{theorem} \label{thm:acdm-convergence}
	After $\ell$ epochs of the ACDM algorithm (equivalently, $(4n
	r^{3/2} + 1)\ell$ iterations), we have
	\[
		\E[g(y_{\ell + 1}) - g(y^*)] \le {1 \over 2^{\ell + 1}}
		(g(y_0) - g(y^*)) 
	\]
\end{theorem}

In the following lemma, we show that the objective function of
(Prox-DSM) satisfies Assumption~{1} in \cite{FR13} and thus the
convergence analysis given in \cite{FR13} can be applied to our
setting.

\begin{lemma} \label{lem:eso}
	Let $R \subseteq \set{1, 2, \dots, r}$ be a random subset of
	coordinate blocks with the property that each $i \in \set{1, 2,
	\dots, r}$ is in $R$ independently at random with probability
	$1/r$.
	Let $x$ and $h$ be two vectors in $\R^{nr}$. Let $h_R$ be the
	vector in $\R^{nr}$ such that $(h_R)^{(i)} = h^{(i)}$ for each
	block $i \in R$ and $(h_R)^{(i)} = 0$ otherwise. We have
		\[ \E\left[g\left(x + h_R \right)\right] \le g(x) + {1 \over
		r} \left\langle \grad g(x), h \right\rangle + {2 \over r}
		\norm{h}^2.\]
\end{lemma}
\begin{proof}
	We have
	\begin{align*}
		\E\left[g\left(x + h_R \right) \right] &= \E\left[r \norm{S(x
		+ h_R)}^2 \right]\\
		&= \E\left[r \norm{Sx}^2 + r \norm{S h_R}^2 + 2r \left\langle
		Sx, Sh_R \right\rangle \right]\\
		&= \E\left[r \norm{Sx}^2 + r \norm{S h_R}^2 + 2r \left\langle
		S^TSx, h_R \right\rangle \right]\\
		&= \E\left[g(x) + \norm{\sum_{i = 1}^r h_R^{(i)}}^2 +
		\left\langle \grad g(x), h_R \right\rangle \right]\\
		&= g(x) + {1 \over r^2} \sum_{i \neq j} \langle h^{(i)},
		h^{(j)} \rangle + {1 \over r} \sum_{i = 1}^r \norm{h^{(i)}}^2
		 + {1 \over r} \langle \grad g(x), h \rangle\\
		&\leq g(x) + {1 \over r^2} \sum_{i \neq j} {1 \over 2} \left(
		\norm{h^{(i)}}^2 + \norm{h^{(j)}}^2 \right) + {1 \over r}
		\sum_{i = 1}^r \norm{h^{(i)}}^2 + {1 \over r} \langle \grad
		g(x), h \rangle\\
		&\leq g(x) + {2 \over r} \sum_{i = 1}^r \norm{h^{(i)}}^2 + {1
		\over r} \langle \grad g(x), h \rangle\\
		&= g(x) + {2 \over r} \norm{h}^2 + {1 \over r} \langle \grad
		g(x), h \rangle
	\end{align*}
\end{proof}

Lemma~\ref{lem:eso} together with Theorem~{3} in \cite{FR13} give us
the following theorem.

\begin{theorem}[Theorem 3 of {\cite{FR13}}] \label{thm:approx}
	Consider iteration $k$ of the APPROX algorithm (see
	Figure~\ref{fig:approx}). Let $y_k = \theta_k^2 u_{k + 1} + z_{k
	+ 1}$. Let $y^* = \argmin_{y \in E} \norm{y - y_k}$ is the
	optimal solution that is closest to $y_k$. We have
		\[ \E[g(y_k) - g(y^*)] \le {4r^2 \over (k - 1 + 2r)^2}
		\left( \left(1 - {1 \over r} \right) (g(z_0) - g(y^*)) + 2
		\norm{z_0 - y^*}^2 \right) \]
\end{theorem}
\begin{proof}
	It follows from Lemma~\ref{lem:eso} that the objective function
	$g$ of (Prox-DSM) and the random blocks $R_k$ used by the APPROX
	algorithm satisfy Assumption~{1} in \cite{FR13} with $\tau = 1$
	and $\nu_i = 4$ for each $i \in \set{1, 2, \dots, r}$. Thus we
	can apply Theorem~{3} in \cite{FR13}.
\end{proof}

Consider an epoch $\ell$. Let $y_{\ell + 1}$ be the solution
constructed by the APPROX algorithm after $4nr^{3/2} + 1$ iterations,
starting with $y_{\ell}$. Let $y^* = \argmin_{y \in E}
\norm{y - y_{\ell + 1}}$ be the optimal solution that is closest to
$y_{\ell + 1}$. Let $\xi_{\ell}$ denote the random choices made
during epoch $\ell$. By Theorem~\ref{thm:approx},
\begin{align*}
	\underset{\xi_{\ell}}{\E}[g(y_{\ell + 1}) - g(y^*)]
	&\le {4r^2 \over (4nr^{3/2} + 2r)^2} \left( \left(1 - {1 \over r}
	\right) (g(y_{\ell}) - g(y^*)) + 2 \norm{y_{\ell} - y^*}^2
	\right)\\
	&\le {1 \over (2nr^{1/2} + 1)^2} \left(g(y_{\ell}) - g(y^*) +
	2\norm{y_{\ell} - y^*}^2 \right)
\end{align*}
We also have
\begin{align*}
	g(y_{\ell}) &= g(y^*) + \langle \grad g(y^*), y_{\ell} -
	y^*\rangle + \int_0^1 \langle \grad g(y^* + t(y_{\ell} - y^*)) -
	\grad g(y^*), y_{\ell} - y^* \rangle dt\\
	&\ge g(y^*) + \int_0^1 \langle \grad g(y^* + t(y_{\ell} - y^*)) - \grad
	g(y^*), y_{\ell} - y^* \rangle dt\\
	&= g(y^*) + \int_0^1 2 tr \norm{S(y_{\ell} - y^*)}^2 dt\\
	&= g(y^*) + r \norm{S(y_{\ell} - y^*)}^2\\
	&\ge g(y^*) + {1 \over n^2 r} \norm{y_{\ell} - y^*}^2 \qquad
	\mbox{(By Theorem~\ref{thm:projection})}
\end{align*}
In the second line, we have used the first-order optimality condition
for $y^*$ (Lemma~\ref{lem:first-order-optimality}). In the last line,
we have used Theorem~\ref{thm:projection}.

Therefore
	\[ \norm{y_{\ell} - y^*}^2 \leq n^2 r (g(y_{\ell}) - g(y^*)) \]
and hence
\begin{align*}
	\underset{\xi_{\ell}}{\E}[g(y_{\ell + 1}) - g(y^*)] &\le {2n^2 r
	+ 1 \over (2nr^{1/2} + 1)^2} \big(g(y_{\ell}) - g(y^*) \big)\\
	&\le {1 \over 2} \big( g(y_{\ell}) - g(y^*) \big)
\end{align*}
Let $\xi = (\xi_0, \dots, \xi_{\ell})$ be the random choices made
during the epochs $0$ to $\ell$. We have
	\[\underset{\xi}{\E}[g(y_{\ell + 1}) - g(y^*)] \le {1 \over
	2^{\ell + 1}} \big( g(y_0) - g(y^*) \big) \]
This completes the proof of Theorem~\ref{thm:acdm-convergence} and
the convergence analysis for the ACDM algorithm.

\section{Experiments}
\label{sec:experiments}

\begin{figure*}[t]
\begin{center}
\subfigure[Penguin]{
\begin{minipage}[t]{0.32\textwidth}
	\centering
	\includegraphics[scale=0.295]{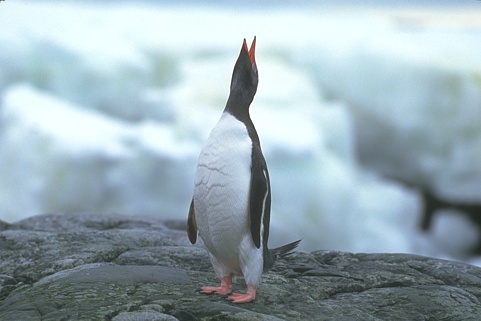}
\end{minipage}
}
\end{center}

\subcapcentertrue
\subcapcenterlasttrue
\subcapnoonelinetrue
\subcapraggedrighttrue

\subfigure[][ACDM~{1} \mbox{$\nu_s = 1.28 \cdot
10^7$} \mbox{$\nu_d = 1.3 \cdot 10^5$}]{
\begin{minipage}[t]{0.16\textwidth}
	\centering
	\includegraphics[scale=0.25]{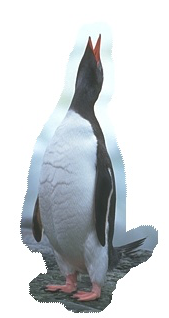}
\end{minipage}
}
\subfigure[][ACDM~{20} \hbox{$\nu_s = 8.38 \cdot 10^6$}
\hbox{$\nu_d = 8.14 \cdot 10^4$}]{
\begin{minipage}[t]{0.16\textwidth}
	\centering
	\includegraphics[scale=0.26]{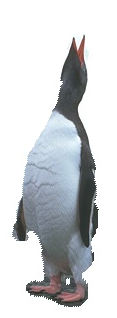}
\end{minipage}
}
\subfigure[][ACDM~{100} \mbox{$\nu_s = 2.9 \cdot 10^6$}
\mbox{$\nu_d = 1.5 \cdot 10^4$}]{
\begin{minipage}[t]{0.16\textwidth}
	\centering
	\includegraphics[scale=0.26]{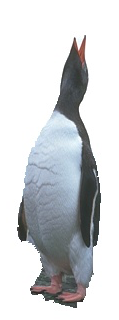}
\end{minipage}
}
\subfigure[][AP~{1} \mbox{$\nu_s = 9.98 \cdot 10^6$} \mbox{$\nu_d
= 1.06 \cdot 10^5$}]{
\begin{minipage}[t]{0.16\textwidth}
	\centering
	\includegraphics[scale=0.25]{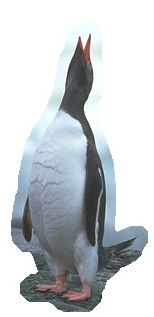}
\end{minipage}
}
\subfigure[][AP~{20} \mbox{$\nu_s = 8.96 \cdot 10^6$} \mbox{$\nu_d
= 1.05 \cdot 10^5$}]{
\begin{minipage}[t]{0.16\textwidth}
	\centering
	\includegraphics[scale=0.25]{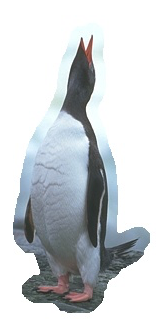}
\end{minipage}
}
\subfigure[][AP~{100} \mbox{$\nu_s = 7.64 \cdot 10^6$}
\mbox{$\nu_d = 8.3 \cdot 10^4$}]{
\begin{minipage}[t]{0.11\textwidth}
	\centering
	\includegraphics[scale=0.25]{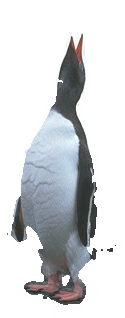}
\end{minipage}
}
\caption{Penguin segmentation results for the fastest (ACDM) and
slowest (AP) algorithms, after 1, 20, and 100 projections. The
$\nu_s$ and $\nu_d$ values are the smooth and discrete dual gaps.}
\label{fig:penguin-segmentations}
\end{figure*}

\begin{figure*}[p]
\subfigure[][Smooth gaps - Octopus] {
\begin{minipage}[t]{0.32\textwidth}
	\centering
	\includegraphics[scale=0.295]{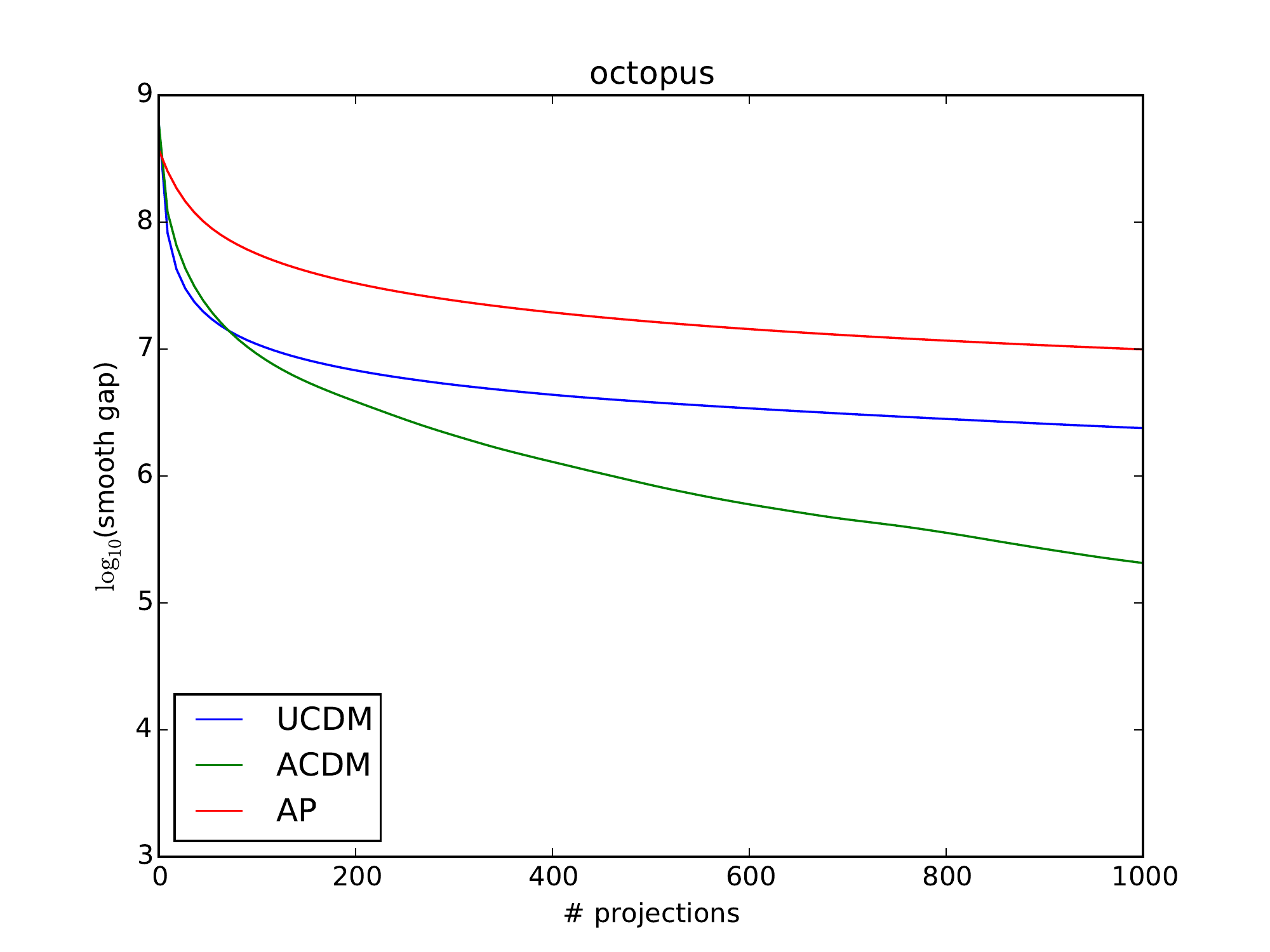}
\end{minipage}
}
\subfigure[][Smooth gaps - Penguin] {
\begin{minipage}[t]{0.32\textwidth}
	\centering
	\includegraphics[scale=0.295]{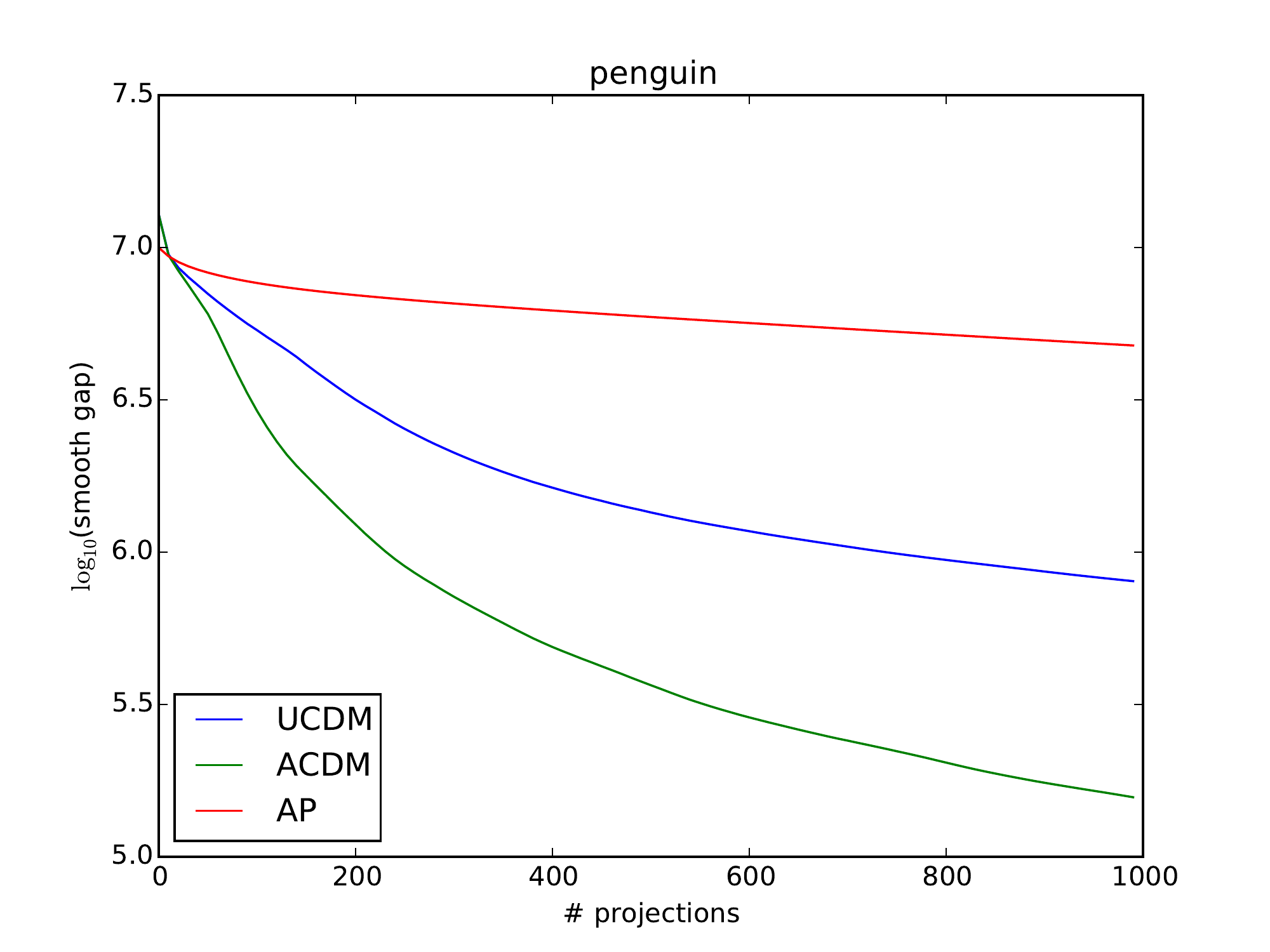}
\end{minipage}
}
\subfigure[][Smooth gaps - Plane] {
\begin{minipage}[t]{0.32\textwidth}
	\centering
	\includegraphics[scale=0.295]{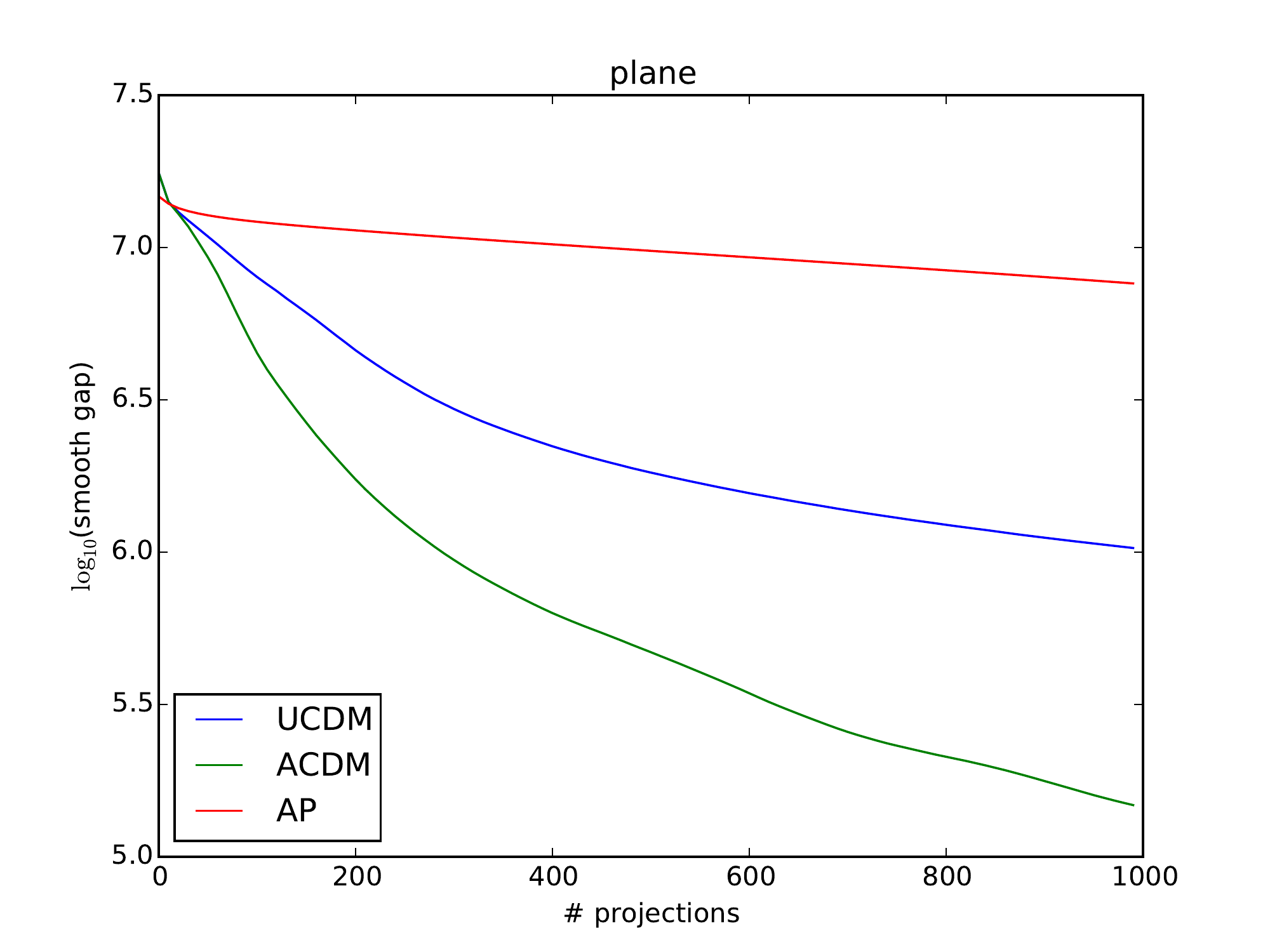}
\end{minipage}
}
\subfigure[][Smooth gaps - Small plant] {
\begin{minipage}[t]{0.32\textwidth}
	\centering
	\includegraphics[scale=0.295]{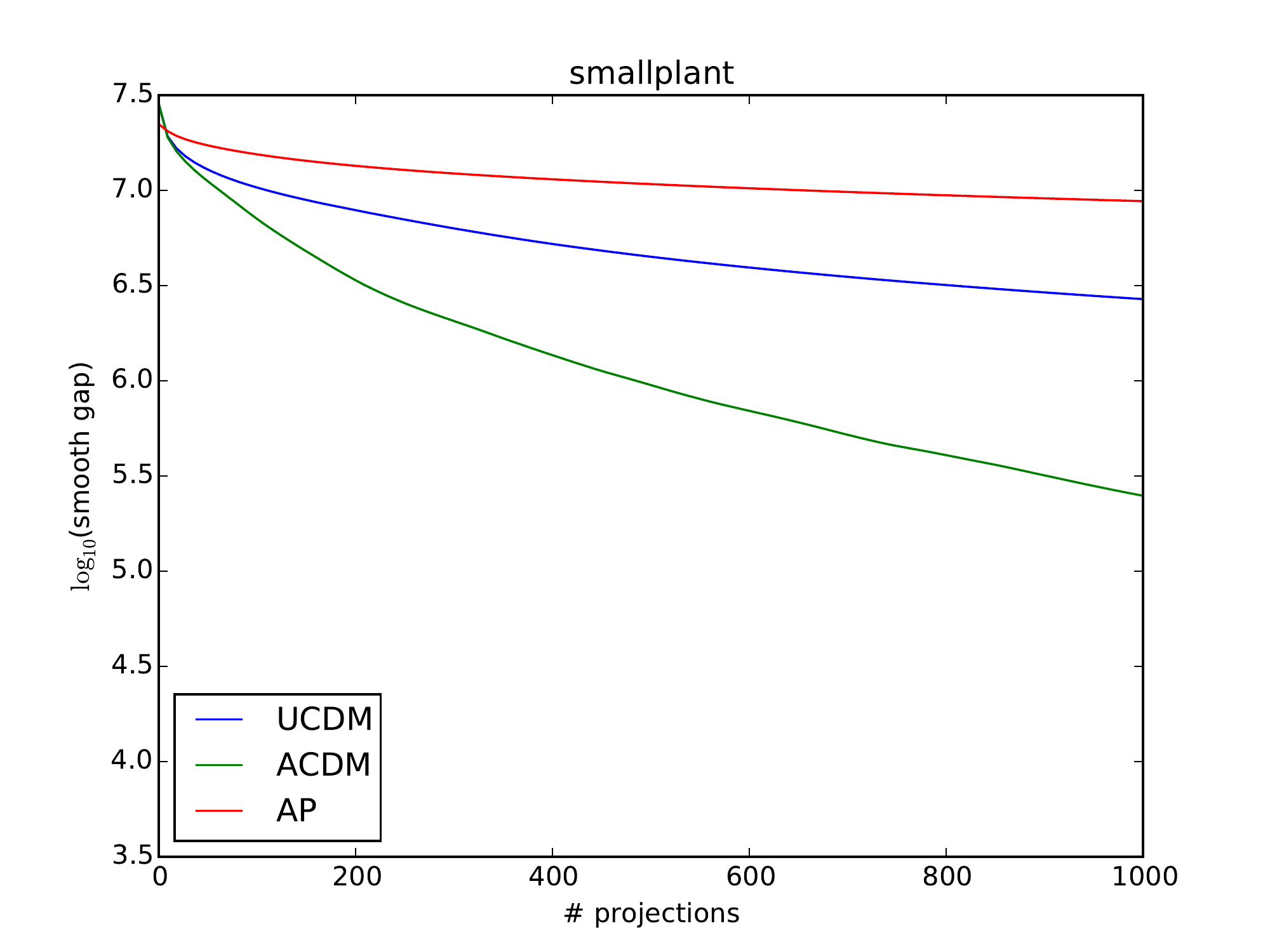}
\end{minipage}
}
\subfigure[][Discrete gaps - Octopus] {
\begin{minipage}[t]{0.32\textwidth}
	\centering
	\includegraphics[scale=0.295]{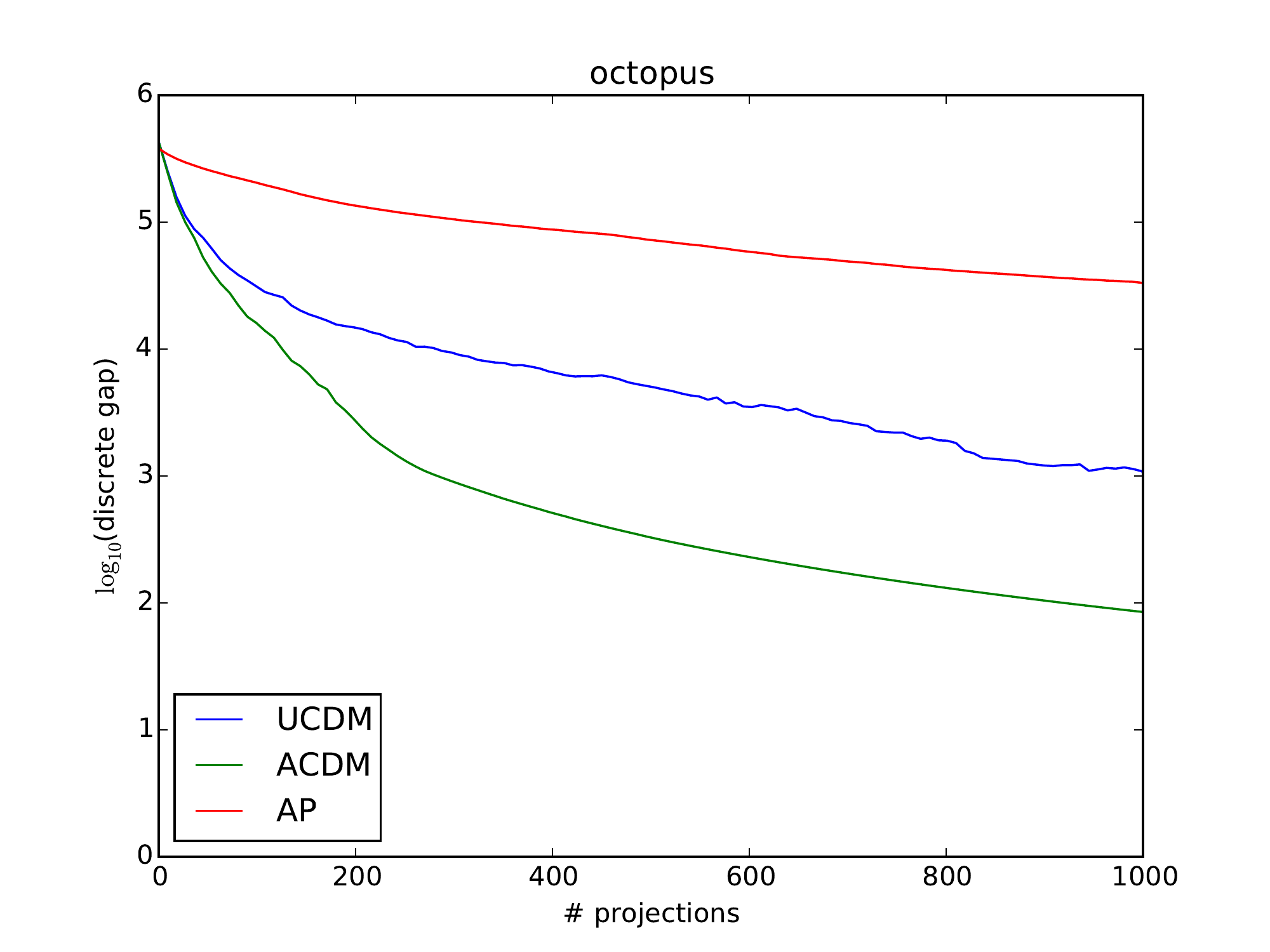}
\end{minipage}
}
\subfigure[][Discrete gaps - Penguin] {
\begin{minipage}[t]{0.32\textwidth}
	\centering
	\includegraphics[scale=0.295]{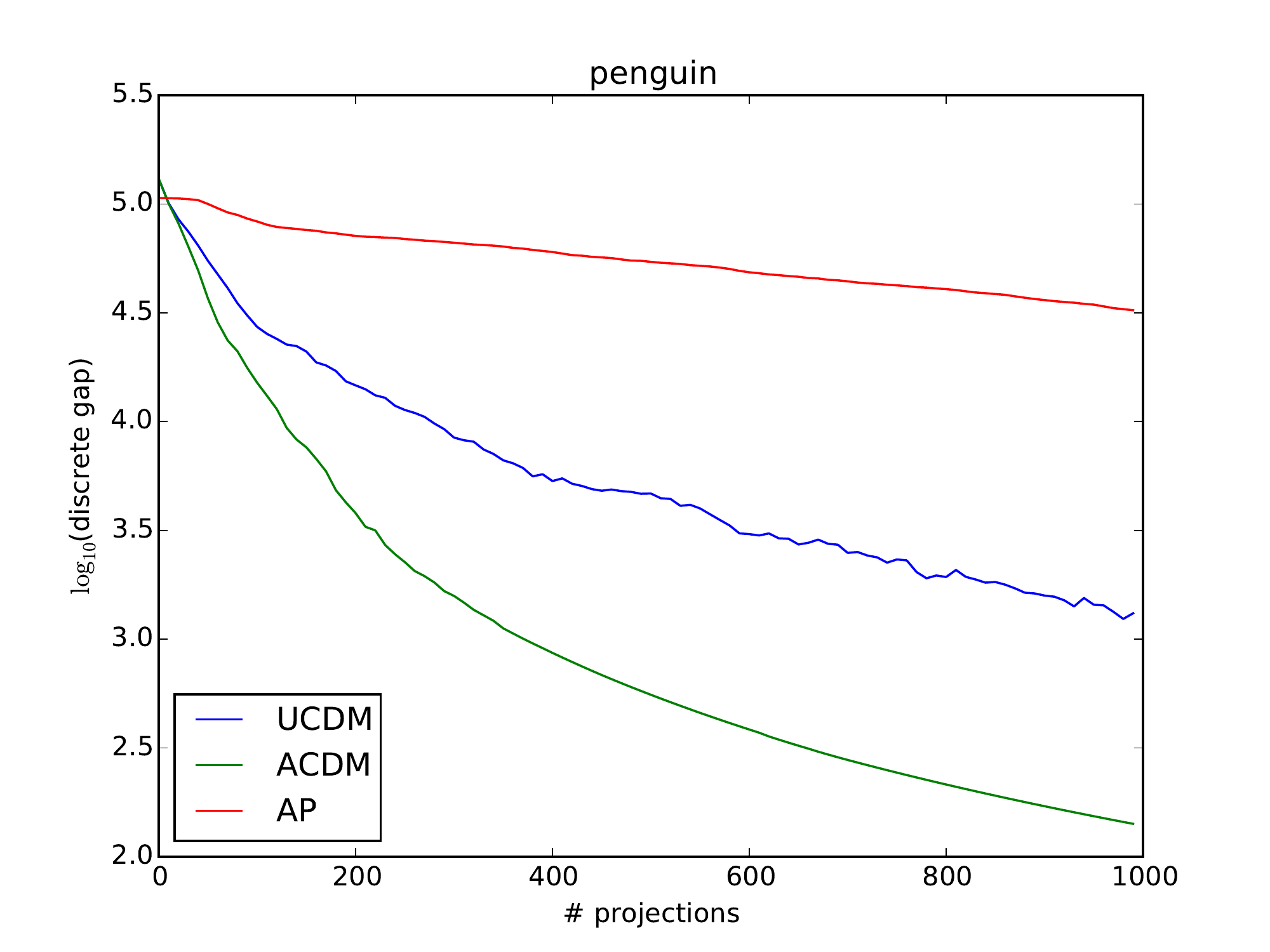}
\end{minipage}
}
\subfigure[][Discrete gaps - Plane] {
\begin{minipage}[t]{0.32\textwidth}
	\centering
	\includegraphics[scale=0.295]{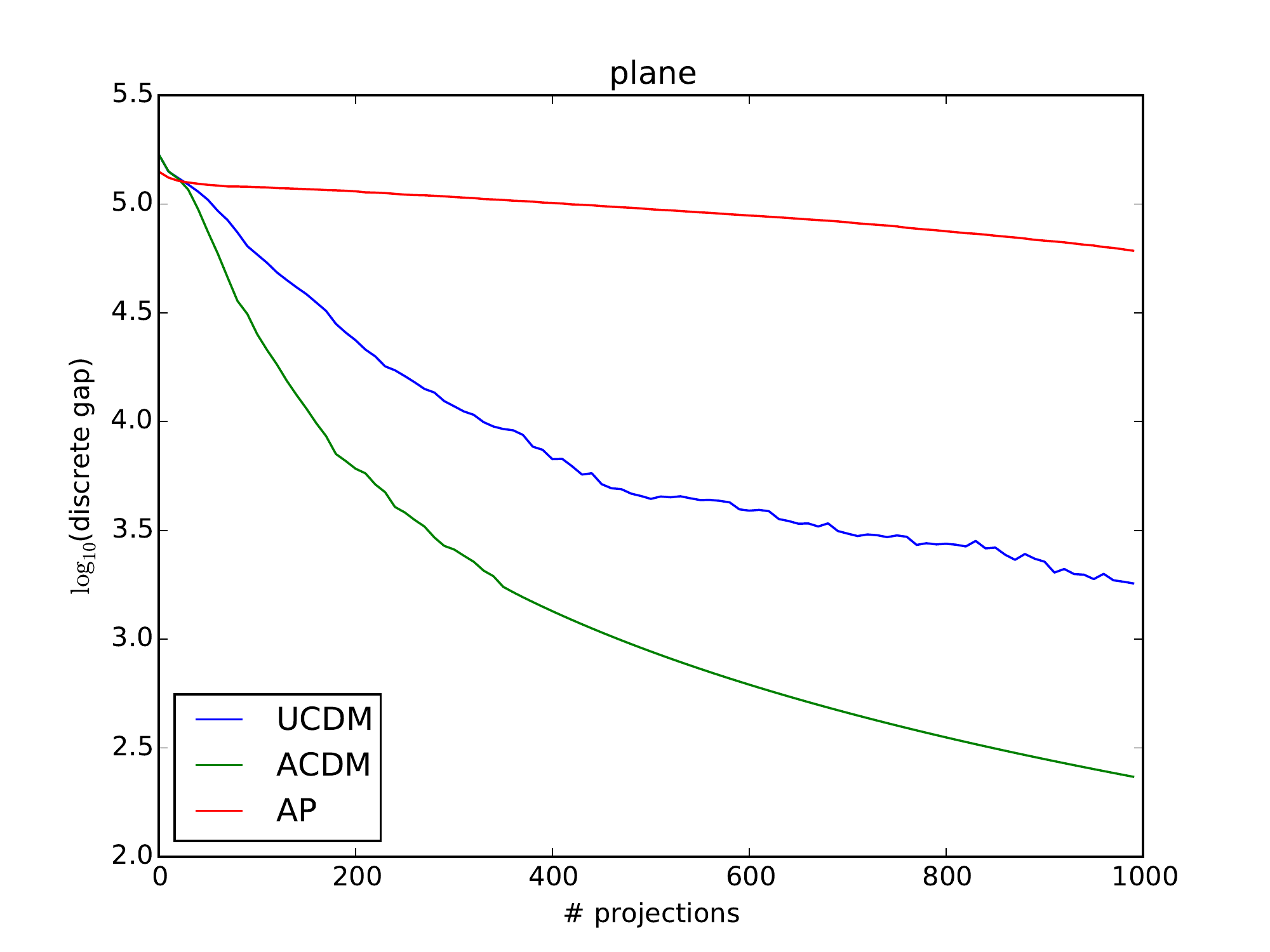}
\end{minipage}
}
\subfigure[][Discrete gaps - Small plant] {
\begin{minipage}[t]{0.32\textwidth}
	\centering
	\includegraphics[scale=0.295]{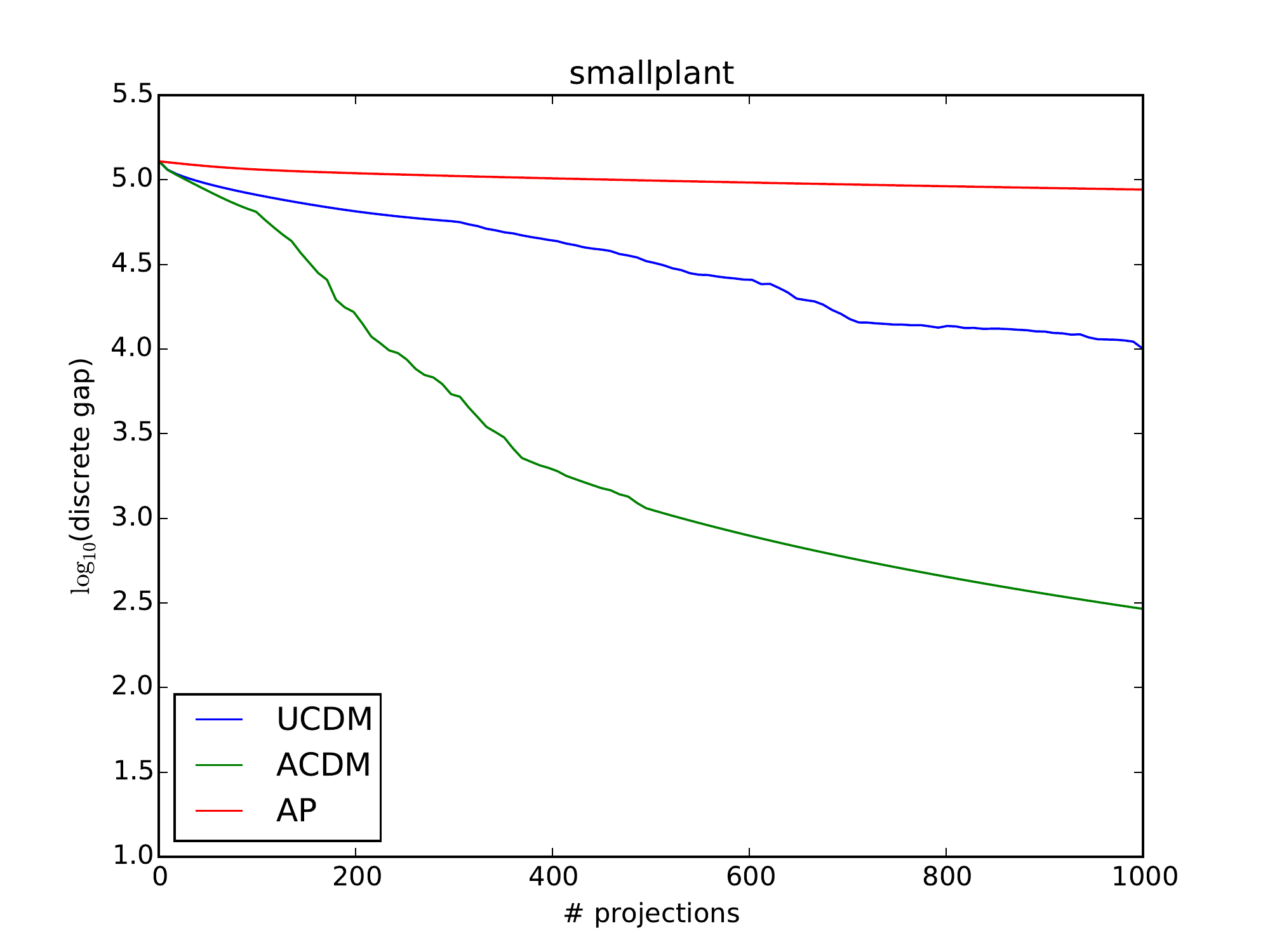}
\end{minipage}
}

\caption{Comparison of the convergence of the three algorithms (UCDM,
ACDM, AP) on four image segmentation instances.}
\label{fig:duality-gaps}
\end{figure*}

\textbf{Algorithms.} We empirically evaluate and compare the
following algorithms: the RCDM described in Section~\ref{sec:rcdm},
the ACDM described in Section~\ref{sec:acdm}, and the alternating
projections (AP) algorithm of \cite{NishiharaJJ14}. The AP algorithm
solves the following best approximation problem that is equivalent to
(Prox-DSM):
\begin{equation} \label{eq:best-approx}
	\min_{a \in \sA, y \in \sY} \norm{a - y}^2
	\tag{Best-Approx}
\end{equation}
where $\script{A} = \set{(a^{(1)}, a^{(2)}, \dots, a^{(r)}) \in
\R^{nr} \colon \sum_{i = 1}^r a^{(i)} = 0}$ and $\script{Y} =
\bigotimes_{i = 1}^r B(F_i)$.

The AP algorithm starts with a point $a_0 \in \sA$ and it iteratively
constructs a sequence $\set{(a_k, y_k)}_{k \geq 0}$ by projecting
onto $\sA$ and $\sY$: $y_k
= \Proj_{\sY}(a_k)$, $a_{k + 1} = \Proj_{\sA}(y_k)$.

$\Proj_{K}(\cdotp)$ is the projection operator onto $K$, that is,
$\Proj_{K}(x) = \argmin_{z \in K} \norm{x - z}$.
Since $\sA$ is a subspace, it is straightforward to project onto
$\sA$. The projection onto $\sY$ can be implemented using the oracles
for the projections $\Proj_{B(F_i)}$ onto the base polytopes of the
functions $F_i$.

For all three algorithms, the iteration cost is dominated by the cost
of projecting onto the base polytopes $B(F_i)$. Therefore the total
number of such projections is a suitable measure for comparing the
algorithms. In each iteration, the RCDM algorithm performs a single
projection for a random block $i$ and the ACDM algorithm performs a
single projection in expectation. The AP algorithm performs $r$
projections in each iteration, one for each block.

{\bf Image Segmentation Experiments.}
We evaluate the algorithms on graph cut problems that arise in image
segmentation or MAP inference tasks in Markov Random Fields. Our
experimental setup is similar to that of \cite{JegelkaBS13}. We set
up the image segmentation problems on a $8$-neighbor grid graph with
unary potentials derived from Gaussian Mixture Models of color
features \cite{RotherKB04}. The weight of a graph edge $(i, j)$
between pixels $i$ and $j$ is a function of $\exp(- \norm{v_i -
v_j}^2)$, where $v_i$ is the RGB color vector of pixel $i$. The
optimization problem that we solve for each segmentation task is a
cut problem on the grid graph.

\emph{Function decomposition:}
We partition the edges of the grid into a small number of
\emph{matchings} and we decompose the function using the cut
functions of these matchings. Note that it is straightforward to
project onto the base polytopes of such functions using a sequence of
projections onto line segments.

\emph{Duality gaps:}
We evaluate the convergence behaviours of the algorithms using the
following measures. Let $y$ be a feasible solution to the dual of the
proximal problem (\ref{eq:proximal}). The solution $x = - \sum_{i =
1}^r y^{(i)}$ is a feasible solution for the proximal problem. We
define the \emph{smooth duality gap} to be the difference between the
objective values of the primal solution $x$ and the dual solution
$y$: $\nu_s = \left(f(x) + {1 \over 2} \norm{x}^2 \right) - \left(-
{r \over 2} \norm{Sy}^2 \right)$.
Additionally, we compute a discrete duality gap for the discrete
problem (DSM) and the dual of its \lovasz relaxation; the latter is
the problem $\max_{z \in B(F)} (z)_{-}(V)$, where $(z)_{-} =
\min\set{z, 0}$ applied elementwise \cite{JegelkaBS13}.  The best
level set $S_x$ of the proximal solution $x = - \sum_{i = 1}^r
y^{(i)}$ is a solution to the discrete problem (DSM).  The solution
$z = - x = \sum_{i = 1}^r y^{(i)}$ is a feasible solution for the
dual of the \lovasz relaxation. We define the \emph{discrete duality
gap} to be the difference between the objective values of these
solutions: $\nu_d(x) = F(S_x) - (-x)_{-}(V)$.

We evaluated the algorithms on four image segmentation
instances\footnote{The data is available at
\url{http://melodi.ee.washington.edu/~jegelka/cc/index.html} and
\url{http://research.microsoft.com/en-us/um/cambridge/projects/visionimagevideoediting/segmentation/grabcut.htm}}
\cite{JegelkaB11,RotherKB04}. Figure~\ref{fig:duality-gaps} shows the
smooth and discrete duality gaps on the four instances.
Figure~\ref{fig:penguin-segmentations} shows some segmentation
results for one of the instances.


\medskip
\textbf{Acknowledgements.} We thank Stefanie Jegelka for providing us
with some of the data used in our experiments.

\clearpage
\bibliographystyle{plain}
\bibliography{submod}

\newpage
\appendix

\section{Proof of Lemma~\ref{lem:dual-prox}}
\label{app:lem-dual-prox}

By the definition of the \lovasz extension, for each $i \in [r]$, we have
	\[ f_i(x) = \max_{y^{(i)} \in B(F_i)} \langle y^{(i)},
	x\rangle.\]
Therefore
\begin{align*}
	&\min_{x \in \R^n} \sum_{i = 1}^r \left( f_i(x) + {1 \over 2r}
	\norm{x}^2 \right) \\&= \min_{x \in \R^n} \sum_{i = 1}^r \left(
	\max_{y^{(i)} \in B(F_i)} \langle y^{(i)}, x \rangle + {1 \over
	2r} \norm{x}^2 \right)\\
	&= \min_{x \in \R^n} \max_{y^{(1)} \in B(F_1), \dots, y^{(r)} \in
	B(F_r)} \sum_{i = 1}^r \left(\langle y^{(i)}, x \rangle + {1
	\over 2r} \norm{x}^2 \right)\\
	&= \max_{y^{(1)} \in B(F_1), \dots, y^{(r)} \in B(F_r)} \min_{x
	\in \R^n} \sum_{i = 1}^r \left( \langle y^{(i)}, x \rangle + {1
	\over 2r} \norm{x}^2 \right)\\
	&= \max_{y^{(1)} \in B(F_1), \dots, y^{(r)} \in B(F_r)} - {1
	\over 2} \norm{\sum_{i = 1}^r y^{(i)}}^2
\end{align*}
On the third line, we have used the fact that the function $\langle
y, x \rangle + (1/2r) \norm{x}^2$ is convex in $x$ and linear in $y$,
which allows us to exchange the $\min$ and the $\max$ (see for
example Corollary~{37.3.2} in Rockafellar \cite{Rockafellar70}). On
the fourth line, we have used the fact that the minimum is achieved
at $x = - \sum_{i = 1}^r y^{(i)}$.

\section{Proofs omitted from Section~\ref{sec:rcdm}}
\label{app:rcdm}

If $x \in \R^{nr}$ and $\sX$ is a subspace of $\R^{nr}$, we let
$\Proj_{\sX}(x)$ denote the projection of $x$ on $\sX$, that is,
$\Proj_{\sX}(x) = \argmin_{z \in \R^{nr}} \norm{x - z}$. We let
$\sX^{\perp}$ denote the orthogonal complement of the subspace $\sX$.

\begin{proposition} \label{prop:project-onto-Q}
	For any point $x \in \R^{nr}$, $\Proj_{\sQ^{\perp}}(x) = S^T S x$
	and thus $\Proj_{\sQ}(x) = x - S^T S x$.
\end{proposition}
\begin{proof}
	Since $\sQ$ is the null space of $S$, $\sQ^{\perp}$ is the row
	space of $S$. Since the rows of $S$ are orthonormal, they form a
	basis for $\sQ^{\perp}$. Therefore, if we let $v_1, \dots, v_n$
	denote the rows of $S$, we have
		\[ \Proj_{\sQ^{\perp}}(x) = \sum_{i = 1}^n \langle x, v_i
		\rangle v_i = S^T S x.\]
\end{proof}

\begin{proposition} \label{prop:optimal-E}
	The set of all optimal solutions to (Prox-DSM) is equal to $E$.
\end{proposition}
\begin{proof}
	We have
	\begin{align*}
		d(\sP, \sQ) &= \min_{y \in \sP} \norm{y - \Proj_{\sQ}(y)}\\
		&= \min_{y \in \sP} \norm{S^T S y} \qquad \Comment{By
		Proposition~\ref{prop:project-onto-Q}}\\
		&= \min_{y \in \sP} \norm{S y}
	\end{align*}
	Since (Prox-DSM) is the problem $\min_{y \in \sP} r \norm{Sy}^2$,
	$E$ is the set of all optimal solutions to (Prox-DSM).
\end{proof}

\begin{proposition} \label{prop:distance-to-Q'}
	Let $y \in \R^{nr}$ and let $p \in E$. We have $d(y, \sQ') =
	\norm{S(y - p)}$.
\end{proposition}
\begin{proof}
	Since $\sQ' = \sQ - v$, we have
	\begin{align*}
		d(y, \sQ') &= d(y + v, \sQ)\\
		&= \norm{\Proj_{\sQ^{\perp}}(y + v)}\\
		&= \norm{S^T S (y + v)}
		\qquad \Comment{By Proposition~{\ref{prop:project-onto-Q}}}\\
		&= \norm{S^T S (y - S^T S p)} \qquad \Comment{Since $v = -
		S^T S p$}\\
		&= \norm{S^T S (y - p)} \qquad \Comment{Since $SS^T = I_n$}\\
		&= \norm{S (y - p)}
	\end{align*}
\end{proof}

\end{document}